\documentclass[twoside]{article}

\usepackage[accepted]{aistats2023}
\usepackage[round]{natbib}

\usepackage{bm}
\usepackage{graphicx}       
\usepackage{amsmath}        
\newcommand{\R}{\mathbb R}
\DeclareMathOperator*{\argmin}{argmin}
\DeclareMathOperator*{\argmax}{argmax}
\usepackage{amssymb}
\usepackage{amsthm}         
\usepackage{xr}
\usepackage{xcolor}
\usepackage{comment}  
\usepackage{booktabs} 
\usepackage{multirow} 

\usepackage{enumitem} 

\usepackage{hyperref}
\usepackage[capitalize]{cleveref} 

\newtheorem{theorem}{Theorem}
\newtheorem{lemma}{Lemma}

\newtheorem{remark}{Remark}
\newtheorem{assumption}{Assumption}

\newtheorem{definition}{Definition}

\newcommand{\bb}{\bm{b}}
\newcommand{\bv}{\bm{v}}
\newcommand{\bw}{\bm{w}}

\newcommand{\bG}{\bm{G}}

\newcommand{\balpha}{\boldsymbol{\alpha}}
\newcommand{\bDelta}{\boldsymbol{\Delta}}

\newcommand{\bpi}{\boldsymbol{\pi}}

\newcommand{\br}{\bm{r}}
\newcommand{\bx}{\bm{x}}

\newcommand{\Eb}[2][{}]{\mathbb E^{#1} \left[ #2 \right]}

\newcommand{\bp}{\bm{p}}

\newcommand{\abs}[1]{\left|#1\right|} 

\begin{document}

%

%

\twocolumn[

\aistatstitle{Balanced Off-Policy Evaluation for Personalized Pricing}

\aistatsauthor{ Adam N. Elmachtoub \And Vishal Gupta \And  Yunfan Zhao }

\aistatsaddress{ Columbia IEOR and DSI \And  USC Marshall \And Columbia IEOR } ]

\begin{abstract}
We consider a personalized pricing problem in which we have data consisting of feature information, historical pricing decisions, and binary realized demand. The goal is to perform off-policy evaluation for a new personalized pricing policy that maps features to prices. Methods based on inverse propensity weighting (including doubly robust methods) for off-policy evaluation may perform poorly when the logging policy has little exploration or is deterministic, which is common in pricing applications. Building on the balanced policy evaluation framework of \citet{kallus2018balanced}, we propose a new approach tailored to pricing applications. The key idea is to compute an estimate that minimizes the worst-case mean squared error or maximizes a worst-case lower bound on policy performance, where in both cases the worst-case is taken with respect to a set of possible revenue functions. We establish theoretical convergence guarantees and empirically demonstrate the advantage of our approach using a real-world pricing dataset.

\end{abstract}


\section{INTRODUCTION}
\label{sec:intro}
Data-driven and personalized pricing has received considerable attention over the past two decades \citep{cohen2017impact, besbes2010testing, ferreira2016analytics, bu2022offline, baardman2019scheduling, wang2021measuring,qi2022offline,biggs2022convex}. Utilizing contextual information in pricing is especially popular due to applications in online shopping \citep{nambiar2019dynamic,elmachtoub2021value}, auto lending \citep{phillips2015effectiveness,ban2021personalized}, air travel \citep{kolbeinsson2022galactic} and beyond \citep{chen2022statistical,wang2021uncertainty, aouad2019market}. The increasing availability of customer data enables personalized pricing strategies. However, experimenting with a new personalized pricing policy that is potentially more profitable or fairer \citep{cohen2022price} can be costly and difficult, motivating the use of off-policy evaluation. Specifically, we study the problem of off-policy evaluation for personalized pricing where feature information such as customer order history, demographics, and market conditions are observed alongside the offered prices and binary purchase decisions. 


There is an extensive literature on off-policy evaluation. Inverse propensity weighting (IPW) and doubly robust (DR) methods are especially popular \citep{dudik2011doubly, hanna2017bootstrapping,swaminathan2015counterfactual,thomas2016data,wang2017optimal,bottou2013counterfactual,athey2021policy}. Both approaches reweight historical data to make the data look as if they were generated by the target policy that we wish to evaluate. While initial research in the area focused on finite, discrete action spaces, more recently \citet{sondhi2020balanced, kallus2018policy,cai2021deep} propose extensions to more general, potentially infinite, action spaces. \citet{biggs2021loss} recasts IPW methods as optimizing a particular loss function and uses this insight to propose suitable generalizations.

Each of the aforementioned methods leverages an approximation of the inverse propensity score to form weights. As noted by \citet{kallus2018balanced}, an inherent shortcoming of such approaches is that when the overlap between the target and logging policy is limited, these methods assign large weights to a small number of data points in the overlap and assign zero weight elsewhere.  This weighting scheme yields high variance estimates, especially on small datasets. In the worst-case when there is zero overlap, IPW methods are not even well-defined.  

While such cases might seem pathological, they are common in pricing applications.
Many real-world firms are reticent to engage in extensive randomized pricing, making limited overlap fairly prevalent.  When firms price deterministically, even simple policy adjustments such as raising all prices 2\% yield zero overlap.  These features make the aforementioned methods less attractive.  

Many authors have proposed general purpose modifications of traditional methods to address these shortcomings.  \citet{elliott2008model,ionides2008truncated,swaminathan2015counterfactual,swaminathan2015self}  each propose various ways to regularize the naive IPW weights, e.g. by clipping large values, to reduce variance.  These methods introduce additional bias into estimates, often in ways that are instance dependent and difficult to quantify.

Other authors attempt to circumvent the issue with IPW by focusing on policy learning -- i.e., identifying a good policy -- rather than policy evaluation.  In cases of zero overlap, \citet{sachdeva2020off} compares three different approaches -- restricting the action space, extrapolating reward, and restricting policy spaces -- and argues in favor of restricting policy spaces.  \citet{kallus2021more} proposes a retargeting approach which reframes the optimal policy as the solution to an alternate off-policy problem with better overlap properties and near-optimal asymptotic variance.  As stated, however, neither approach directly addresses policy evaluation.  Insofar as firms are often interested in the performance of a specific, target pricing policy that may have been chosen for qualitative, business-specific reasons, there remains a need for effective policy evaluation methods that balance bias and variance and provide provable performance guarantees.

Inspired by the balanced policy evaluation method of \citet{kallus2018balanced}, we propose an alternate approach to off-policy evaluation for pricing applications. Like IPW and DR methods, we estimate the performance of the policy by a weighted average of the historical data points. However, unlike these methods,
we use weights that  either \textit{(i)} minimize the worst-case mean squared error of our estimated revenue or \textit{(ii)} maximize a worst-case lower bound on the unknown target revenue. In both cases, the worst-case is taken over a set of plausible revenue functions.

Our work differs from \citet{kallus2018balanced} in three critical aspects: \textit{(i)} We focus on a binary demand  response variable rather than a continuous one with a homoscedastic variance. Binary demand induces a more complex form for the variance of our estimator and consequently complicates the worst-case optimization problem defining our weights. By contrast, the corresponding optimization in \citet{kallus2018balanced} is an unconstrained, convex quadratic program with a closed-form solution.  \textit{(ii)} Although we treat worst-case mean squared error (MSE) (similar to \citet{kallus2018balanced}), firms are also concerned with operational criteria such as a guaranteed lower bound on revenue. We show how our approach can be modified to compute such a lower bound (via Bernstein's inequality) and contrast the behavior of the resulting estimator with the MSE approach.  \textit{(iii)} \citet{kallus2018balanced} focuses primarily on the case of a small number of discrete actions, while typical pricing problems involve continuous action spaces. 
In particular, one cannot apply \citet{kallus2018balanced} ``out-of-the-box" to continuous action spaces, since the approach assumes no structure across actions (prices) and would thus yields overly conservative estimates (\cite{kallus2020comment} suggests a way to address this). By contrast, we enforce smoothness of the demand function across prices by assuming this revenue function belongs to a particular reproducing kernel Hilbert space (RKHS).

{\bf Paper Outline:} We start in Section~\ref{sec:Notation} with the notation and setup, followed by an analysis of weighted revenue estimators in Section~\ref{sec:weighted_rev_estimators}. We present our off-policy evaluation approach in Section~\ref{sec:optimizing_worst_case_metrics}. In Section~\ref{sec:theory}, we establish theoretical guarantees for our approaches. We present experimental results on both synthetic datasets and a real world pricing dataset in Section~\ref{sec:exps}. We describe heuristics for estimating parameters in Section \ref{sec:heuristics} and conclude in Section \ref{sec:conclusion}. 




\section{NOTATION AND MODEL} \label{sec:Notation}
We assume the following (fixed-design) data generation mechanism:  We are given a set features $\bx_i \in \mathcal X$ for $i=1\ldots, n$.  Price-demand pairs are distributed as
\begin{align*}
	P_{i} &\sim g_0(\cdot, \bx_i),  &&i =1, \ldots, n,
	\\
	D_i \mid P_{i} &\sim \text{Bernoulli}(d(\bx_i, P_{i})),  &&i =1, \ldots, n,
\end{align*}
for some unknown demand function $d(\cdot, \cdot)$  that maps features and prices to $[0,1]$.  Here the density $g_0(\cdot, \cdot)$ encodes our logging pricing policy,  i.e. we draw a random price from density $g_0(\cdot,\bx)$ when presented with a feature $\bx$.  When the logging policy is deterministic, we interpret $g_0(\cdot, \bx)$ as a Dirac delta function.  

Our dataset $\{(\bx_i,p_i,d_i) \subseteq \mathcal X \times \R_+ \times \{0,1\}: i \in [n] \}$ consists of single a realization of this process.

Loosely, our goal is to evaluate a target policy that draws a random price from the density $g_1(\cdot, \bx)$ when presented with feature $\bx$.  Formally, let
\begin{align*}
P_{n+i} &\sim g_1(\cdot, \bx_{i})  &&i =1, \ldots, n,
\end{align*}
and let $p_{n+i} \in \R$ for $i\in [n]$ be a corresponding realization.  Then, if we define the expected revenue function $r(\bx, p) := pd(\bx,p)$, the expected revenue under the target policy is
\begin{align*} \tag{Target Revenue}
	\mathcal R &:= \textstyle \frac{1}{n} \sum_{i=1}^{n} p_{n+i}d(\bx_i, p_{n+i}) 
\\	& \ = \ 
	\textstyle \frac{1}{n} \sum_{i=1}^{n} r(\bx_i, p_{n+i}), 
\end{align*}
which we emphasize is a constant. Our goal is to estimate and provide high confidence  bounds on this constant.  

We stress that, in what follows, our method does not require explicit knowledge of $g_0(\cdot, \cdot)$ or $g_1(\cdot, \cdot)$.

In keeping with the literature on doubly-robust estimators, we define a reference revenue function:
\begin{definition}[Reference Revenue Function]
The revenue function can be written as $r(\cdot, \cdot) = \hat r(\cdot, \cdot) + \Delta(\cdot, \cdot)$, for a known reference revenue $\hat r(\cdot, \cdot)$, and a perturbation function $\Delta(\cdot, \cdot)$.
\end{definition}
This decomposition is without loss of generality (take $\hat r(\cdot, \cdot) = 0$).  In practice, we may have a good reference model $\hat r(\cdot, \cdot)$ that we believe reasonably captures the revenue curve.  Thus, the estimators are best thought of as a perturbation to this reference.  

To streamline notation, we define $\bp \in \R^{2n}$ to be the vector of prices $p_1,...,p_{2n}$. Similarly, we define the vectors $\br, \hat\br, \bm \Delta \in \R^{2n}$ such that for for $i\in [2n]$,
\begin{align*}
r_i  = r(\bx_i, p_i), \ \  
\hat r_i  = \hat r(\bx_i, p_i), \ \ 
\Delta_i  = \Delta(\bx_i, p_i). 
\end{align*}

We focus on the doubly robust weighted revenue estimator
\begin{equation}
	\hat{\mathcal R}(\bw) := 
	\frac{1}{n}\sum_{i=1}^n w_i (p_i D_i - \hat r_i) + 
	\frac{1}{n}\sum_{i=n+1}^{2n} \hat r_i ,
    \label{eq:revenue_estimator}
	   \tag{Estimator}
\end{equation}
for some weights $\bw$ that we will specify.
\section{PROPERTIES OF WEIGHTED REVENUE ESTIMATORS}
\label{sec:weighted_rev_estimators}
We first introduce  general properties of weighted revenue estimators with honest weights, i.e., the weights are independent of demand realizations.  These properties depend on the vector $\br$ which in practice is unknown. Nonetheless, these properties  
serve as a building block for our approach later on where we take a worst-case perspective on $\br$. 

\subsection{Mean Squared Error}
Define 
\small
\begin{align}  \notag 
 &\text{MSE}(\bw, \br) := \Eb{ (\mathcal R - \hat{\mathcal R}(\bw) )^2} \\ \label{eq:DumbMSE}
=& \Eb{\left(\mathcal R - \frac{1}{n}\sum_{i=n+1}^{2n} \hat r_i - \frac{1}{n}\sum_{i=1}^n w_i (p_i D_i  - \hat r_i)\right)^2}.	
\end{align}
\normalsize

Note $\mathcal R$ and $\Eb{p_j D_j}$ depend on the unknown revenue $\br$.
In Lemma \ref{lem:BiasVar} below, we provide a more explicit expression for the MSE, which takes into account the binary nature of  demand.  (See Appendix \ref{sec:proofs2} for proofs.)

\begin{lemma}[Bias and Variance Decomposition]\label{lem:BiasVar}
Let
\begin{align*}
	\bb(\bw) &:= \frac{1}{n}\left(w_1, \ldots, w_n, -1, \ldots, -1 \right)^\top \in \R^{2n}
\\
	\bv(\bw) &:= \frac{1}{n^2}\left( w_1^2 p_1, \ldots, w_n^2 p_n, 0, \ldots, 0\right)^\top \in \R^{2n} \\
\end{align*}
Then, we have
\begin{align*}
	 Bias(\bw, \br) &:=  \Eb{\hat{\mathcal R}(\bw) - \mathcal R}  = \bb(\bw)^\top \left(\br - \hat\br\right), \\
	 Var(\bw, \br) &:= \Eb{ \left(\hat{\mathcal R}(\bw) - \Eb{\hat{\mathcal R}(\bw)}\right)^2} \\
	&=\bv(\bw)^\top\br - 
	\frac{1}{n^2}\br^\top \begin{pmatrix} \text{diag}(w_1^2, \ldots, w_n^2) & \bm 0 
	\\ \bm 0 & \bm 0
			 \end{pmatrix} \br,
\end{align*}
and, of course,  $\text{MSE}(\bw, \br) =   Bias(\bw, \br)^2 +  Var(\bw, \br).$
\end{lemma}



\subsection{High-Probability Bound} 
\label{sec:high_probability_bounds}
We next provide a high-confidence lower bound on the true revenue $\mathcal R$ in terms of the estimate $\hat{\mathcal R}(\bw)$. From an operational perspective, lower bounds provide  ``safe" guarantees on potential revenue.  Although similar techniques could be used to form upper bounds, they are less useful in practice.

Define 
\begin{align} \notag
    Bern(\bw,\br) &:= \bb(\bw)^\top (\br - \hat\br) + \sqrt{2 Var(\bw; \br)\log(1/\epsilon)} 
	 \\ \label{eq:BernsteinBound}
  & \quad+ \frac{1}{3n} \max_{1 \leq i \leq n} |w_i| p_i\log(1/\epsilon).
\end{align}

\begin{lemma}[Revenue Lower Bound]
\label{lem:rev_bounds}
With probability at least $1-\epsilon$ over the realization of $(D_1, \dots, D_n)$, we have that $\mathcal R \geq \hat{\mathcal R}(\bw) - \text{Bern}(\bw,\br)$. 
\end{lemma}
The lemma is a direct application of Bernstein's inequality.

\begin{remark}[Convexity in $\bw$] \label{rem:ConvexityofRevBounds}
\label{remark:mse_convex_in_w}
Since the expectation of a convex function is convex, Eq.~\eqref{eq:DumbMSE} shows the map $\bw \mapsto  \text{MSE}(\bw, \br)$ is convex in $\bw$ for a fixed $\br$. Similarly, the function $\text{Bern}(\bw, \br)$ is convex in $\bw$ for a fixed $\br$ since $\sqrt{Var(\bw; \br)} = \sqrt{\frac{1}{n^2}\sum_{i=1}^n w_i^2 r_i(p_i - r_i)}$ is a weighted $\ell_2$-norm, and, hence, $\bw \mapsto \text{Bern}(\bw, \br)$ is a sum of convex functions.  We will leverage these convexity properties when formulating optimization problems to compute our weights. 
\end{remark}

\section{A BALANCED APPROACH FOR OFF-POLICY EVALUATION IN PRICING} 
\label{sec:optimizing_worst_case_metrics}
The expression for mean squared error and the lower bound in Eq.~\eqref{eq:BernsteinBound} depend on the unknown revenue vector $\br$. 
Our approach will be to compute weights $\bw$ that optimize these metrics over ``plausible" worst case realizations of $\br$. To define ``plausible," we make the following assumption for the remainder of the paper:
\begin{assumption}[Perturbation Function is in RKHS]  \label{asn:RevRKHS}
There exists an RKHS $\mathcal H$ with kernel $K(\cdot, \cdot)$ and norm $\| \cdot \|_{\mathcal H}$ such that 
$\Delta(\cdot, \cdot) \in \mathcal H$ and 
$\| \Delta(\cdot, \cdot) \|_{\mathcal H} < \infty$. \end{assumption}

Assumption~\ref{asn:RevRKHS} asserts that the unknown perturbation function is ``smooth" in the sense that it has a bounded RKHS norm.   By suitably choosing the kernel $K(\cdot,\cdot)$, we can enforce structural constraints on $\Delta(\cdot,\cdot)$, e.g., that $\Delta(\cdot,\cdot)$ is linear in price or Sobolev smooth in the covariates. See \citep{smola2004tutorial} for details. 

For notational convenience, we let $\Gamma := \| \Delta(\cdot, \cdot) \|_{\mathcal H}$.  Define the Graham matrix $\bG \in \R^{2n \times 2n}$ 
by 
\[
	\bG_{ij} := K\left((\bx_i, p_i), (\bx_j, p_j)\right) \qquad 1 \leq i, j \leq 2n.
\]

Under Assumption~\ref{asn:RevRKHS}, the Representer Theorem \citep{wahba1990spline} implies that there exists $\balpha \in \R^{2n}$ such that 
 \begin{align*}
	\bDelta = \bG \balpha  \quad \text{ and }  \quad 
	\balpha^\top \bG \balpha = \Gamma^2.
\end{align*}
We further make the following common assumption.
\begin{assumption}
\label{asn:gram_matrix_invertible}
The Graham matrix $\bG$ is invertible. 
\end{assumption}
From Assumptions \ref{asn:RevRKHS} and \ref{asn:gram_matrix_invertible},
\begin{align}
\bDelta^\top \bG^{-1} \bDelta = \Gamma^2.
\label{eq:rkhs_constraint}
\end{align}

Since $d(x,p)\in [0,1]$ for any $x,p$, we have 
\begin{align}
\bm 0 \leq   \br  \leq \bp.
\label{eq:box_constraint}
\end{align}
Combining \eqref{eq:rkhs_constraint} and \eqref{eq:box_constraint}, we seek weights that minimize a worst-case metrics $\phi(\bw, \br)$ over plausible revenue functions:
\begin{align}
&\bw^* \in \argmin_{\bw} \max_{\br} \quad \phi(\bw, \br) \label{eq:genertic_obj}
\\
&\text{s.t.} \quad  \bm 0 \leq \br \leq \bp \ ,  \quad 
 (\br - \hat \br)^\top \bG^{-1}(\br-\hat \br) \leq \hat \Gamma^2.\nonumber
\end{align}
Here $\phi(\bw, \br)$ can be $\text{MSE}(\bw, \br)$ or $\text{Bern}(\bw, \br)$. We denote the corresponding solutions by $\bw^{MSE}$ and $\bw^B$, respectively.

Since in practice, we do not know the ground truth $\Gamma$, we proxy $\Gamma$ by user-specified constant $\hat \Gamma$ in Problem~\ref{eq:genertic_obj}.  (We discuss heuristics for estimating $\hat \Gamma$ in Section~\ref{sec:heuristics}.) 


\begin{remark}[Unconstrained Weights]
   In contrast to \citet{kallus2018balanced}, 
we do not impose an additional simplex constraint on the weights.  Indeed, the value of the target policy need not be on the same order of magnitude as the logging policy, e.g., when we raise price significantly. Thus, an ideal set of weights might not satisfy such a constraint. That said, our Bernstein variant ($\bw^B)$ does regularize away from overly large weights via the weighted $\ell_\infty$ norm in Eq.~\eqref{eq:BernsteinBound}.  This regularization emerges naturally via the probabilistic analysis rather than being imposed via an artificial simplex, or normalizing, constraint. 
\end{remark}



\begin{remark}[Honest vs. Dishonest Weights]
When $\hat \br$, $\hat\Gamma$, and the kernel $K(\cdot,\cdot)$ are specified exogenously, i.e., independently of the demand realizations, both $\bw^{MSE}$ and $\bw^B$ are honest. We study the corresponding estimators $\hat R(\bw^{MSE})$ and $\hat R(\bw^B)$ theoretically in  Section~\ref{sec:theory}. 

In practice, we suggest fitting these parameters to the data via the heuristics in Section~\ref{sec:heuristics}.  The resulting weights are ``dishonest."  While it might be possible to extend our theoretical results to this setting by assuming that $(\hat r, \hat \Gamma, K(\cdot, \cdot))$ are chosen from a suitably low-complexity class, we do not pursue this theoretical analysis here. Rather, we present numerical evidence in Sec.~\ref{sec:exps} that  even with dishonest weights, our estimator performs well.
\end{remark}

\subsection{Solution Approach} \label{sec:GenericAlg}
We next discuss how to solve  \eqref{eq:genertic_obj}. 

For a fixed $\bw$, consider the inner problem of finding the worst case (WC) revenue:
\begin{align*}
& \br^{WC}(\bw) := \argmax_{\br} \quad \phi(\bw, \br) \\
& \text{s.t.} \quad  \bm 0 \leq \br \leq \bp \ ,  \quad 
 (\br - \hat \br)^\top \bG^{-1}(\br-\hat \br) \leq \hat \Gamma^2.
\end{align*}

Let $h(\bw) = \phi(\bw, \br^{WC}(\bw))$. 
Since $\bw \mapsto \phi(\bw, \br)$ is convex for each $\br$ by Remark~\ref{rem:ConvexityofRevBounds}, Danskin's Theorem \citep{bertsekas1997nonlinear} shows that $h(\bw)$is in fact convex in $\bw$, and, when $\br^{WC}(\bw)$ is the unique optimizer, 

\[
	\nabla h(\bw) = \nabla_{\bw} \phi(\bw, \br)|_{\br=\br^{WC}(\bw)}.
\]
Thus, we can minimize $h(\bw)$ using any number of gradient-based algorithms. (In our numerical experiments, we use the a first order trust-region method from scipy.optimize.) Evaluating a gradient requires determining $\br^{WC}(\bw)$, i.e., solving the inner problem.

That said, for large $n$, computing gradients in the Bernstein objective is perhaps easier than for the MSE objective.  For the Bernstein objective, the inner maximization problem can be reformulated as a concave quadratic maximization problem in $\br$ (see Appendix~\ref{sec:reformulation_bern}). By contrast, for the MSE objective, the inner problem is an in-definite quadratic programming problem. Such problems can, in the worst-case, be NP-Hard, but are often practically solvable with modern solvers for moderate sized instances. 
In  our experiments, we use Gurobi for both computations. 
\section{THEORETICAL RESULTS}
\label{sec:theory}
Recall our approach to off-policy evaluation for pricing applications is partially motivated by the observation that in typical pricing applications,
the overlap between the logging and evaluation policies may be small since both policies may entail
little randomization. This feature precludes the use of methods based on inverse propensity scores
that require sufficient overlap, including doubly-robust methods.

In this section we establish a ``sanity-check'' result, i.e., that when sufficient overlap does exist, our method
achieves convergence rates similar to the doubly-robust methods.

\begin{assumption}[Overlap] \label{overlap}
For all $(p, \bx) \in \R_+ \times \mathcal X$, if $g_0(p, \bx) = 0$ then $g_1(p, \bx) = 0$.
\end{assumption}

From Assumption \ref{overlap}, the inverse propensity (IP) weights 
\begin{equation}\label{eq:InversePropScore}
W_i^{IP} := \frac{g_1(P_i, \bx_i)}{g_0(P_i, \bx_i)}
\end{equation}
are well-defined for all $i = 1, \ldots, n$.  

\subsection{Mean Squared Error}
\label{subsec:theory_mse}
We first consider $\bw^{MSE}$ and the corresponding estimator $\hat{\mathcal R}(\bw^{MSE})$.  Theorem \ref{thm:mse_asymptotics} shows that true (unknown) MSE of this estimator converges to zero at a rate of $\frac{1}{n}$, despite not knowing $g_0(\cdot, \cdot)$, $g_1(\cdot, \cdot)$ or $\Gamma$.  (See Appendix~\ref{sec:mseproof} for proof.) 

For convenience, let $Z_i = (\bx_i, P_i)$ for $i =1, \ldots, 2n$.

\begin{theorem}[Convergence of MSE]\label{thm:mse_asymptotics}
Suppose that 
\begin{enumerate}[label=\roman*)]
\item $\frac{1}{n} \sum_{i=1}^n \Eb{\left(W_i^{IP} -1\right)K(Z_{n+i}, Z_{n+i})}= O(1)$ 
\item $\frac{1}{n} \sum_{i=1}^n \Eb{(W_i^{IP}P_i)^2} = O(1)$ 
\end{enumerate}
Then, under Assumptions~\ref{asn:RevRKHS}, ~\ref{asn:gram_matrix_invertible}, and ~\ref{overlap}, we have
$	\text{MSE}(\bw^{MSE}, \br)= O_p\left(\frac{1}{n}\right).
$
\end{theorem}
For clarity, the ``probability" in \cref{thm:mse_asymptotics} is taken over the randomness in both $\{D_i : i\in [n]\}$ and $\{P_i : i \in [2n]\}.$

To help develop intuition around the assumptions of the above theorem, consider the case where $K(\cdot, \cdot)$ is the gaussian kernel, so that $K(Z_{n+i}, Z_{n+i})$ is almost surely a constant.  Then the first condition $i)$ holds trivially since $\Eb{W^{IP}_i } = 1$ by construction.  The second condition $ii)$ essentially requires that 
for a typical point, the inverse propensity score weights are not too large -- they are $O(1)$.  This requirement is analogous to requiring sufficient overlap between the logging and evaluation policies, since $W^{IP}$ explodes as the overlap shrinks.  In this sense, \cref{thm:mse_asymptotics} is a ``sanity-check" result.


\subsection{Bernstein Bound}
\label{subsec:theory_bernstein_bound}

We next consider $\bw^B$ and corresponding estimator $\hat{\mathcal R}(\bw^B)$.
Recall Lemma~\ref{lem:rev_bounds} shows that, with high probability, $\hat{\mathcal R}(\bw^B) - \text{Bern}(\bw^B, \br)$ lower bounds the true (unknown) revenue.  We will next show that this lower bound is not too loose, specifically, that $\text{Bern}(\bw^B, \br) = O_p(1/\sqrt n)$.  
(See Appendix~\ref{sec:bernstein} for proof.)

\begin{theorem}[Safe Guarantee]\label{thm:bernstein_asymptotics}
Suppose that 
\begin{enumerate}[label=\roman*)]
\item $\frac{1}{n} \sum_{i=1}^n \Eb{\left(W_i^{IP} -1\right)K(Z_{n+i}, Z_{n+i})}= O(1)$ 
\item $\frac{1}{n} \sum_{i=1}^n \Eb{(W_i^{IP}P_i)^2} = O(1)$ 
\end{enumerate}
Then, under Assumptions~\ref{asn:RevRKHS}, ~\ref{asn:gram_matrix_invertible}, and ~\ref{overlap}, we have
$	\max\left(0, \text{Bern}(\bw^B, \br)\right) = O_p\left(\frac{1}{\sqrt n}\right).
$
\end{theorem}
In other words, the unknown  true revenue cannot exceed our estimate by more than $O_p(1/\sqrt n)$.  In this sense, our estimate provides a ``safe" guarantee that is not too loose.

\begin{remark}[One-Sided vs. Two-Sided Bounds]
In Theorem~\ref{thm:bernstein_asymptotics}, we obtain a one-sided convergence result  because we used a one-sided probability bound to define $\text{Bern}(\bw, \br)$.  
If one sought a stronger two-sided convergence, one could instead introduce an absolute value in \cref{eq:BernsteinBound} and define the corresponding estimator.  

In our numerical experiments, we found this ``two-sided" estimator performs worse than our proposed one-sided estimator. Hence we have chosen to only present theoretical results for the one-sided estimator.  
\end{remark}

\section{NUMERICAL RESULTS}
\label{sec:exps}
We describe our numerical results, but please see our GitHub for for reproducibility code and documentation. \footnote{\url{https://github.com/yzhao3685/pricing-evaluation}} 

\subsection{Mean Squared Error}
\label{sec:exps_mse}
We first study $\bw^{MSE}$ and corresponding estimator $\hat{\mathcal R}(\bw^{MSE})$.  We denote our corresponding method as BOPE-B for ``Balanced Off-Policy Evaluation for Binary response." 

We compare the performance of the following methods on synethetic and real-world datasets:
\begin{itemize}
    \item (LASSO) A ``direct" regression estimator corresponding to $\hat{\mathcal R}(\bm 0)$. This linear regression method with $\ell_1$ penalty predict the demand $d(\cdot,\cdot)$, and revenue is obtained from multiplying it by the price. This serves as a baseline. 
    \item (SPPE) Semi-parametric policy evaluation \citep{chernozhukov2019semi} which is an extension of the classical DR method to a setting where the dependence of the policy value on the treatment is known.  In pricing applications, this amounts to specifying  a priori how demand depends on price. In our experiments, we apply the method assuming demand is linear in price.
    
    \item (BOPE) The Balanced Off-Policy Evaluation method of \cite{kallus2018balanced}.
    This method can be seen as an instance of Problem~\ref{eq:genertic_obj} with $\phi(\bw,\br)= Bias^2(\bw,\br) + \frac{1}{n^2}\sigma^2\sum_{i=1}^n w_i^2$
    for some user-defined $\sigma^2$.  Loosely, this objective is the worst-case mean squared error if $p_i D_i$ were homoscedastic random variables with variance $\sigma^2$ and mean $r(p_i, \bx_i)$.  Thus, this method does not exploit the binary structure of demand.
      We select hyperparameters according to the heuristic proposed in \cite{kallus2018balanced} (see \cref{sec:heuristics}).
    \item (BOPE-B) Our proposed Balanced Off-Policy Evaluation estimator for Binary response, $\hat{\mathcal R}(\bw^{MSE})$, with hyper-parameters chosen according to heuristics described in Section~\ref{sec:heuristics}.
\end{itemize}
For each of BOPE-B, BOPE, and LASSO, we use a LASSO linear regression to estimate $\hat r(\cdot, \cdot)$.  



Before delving into the details of the experiments, we summarize our main findings:
\begin{itemize}
    \item By exploiting the binary nature of demand, the BOPE-B estimator generally has an advantage over the BOPE estimator, and substantive advantage of the SPPE estimator.  
    \item When the baseline LASSO, itself, has small MSE, there is little room for improvement and both BOPE and BOPE-B perform comparably.  When the baseline estimate is poor, both BOPE and BOPE-B peform substantively better than baseline.
    \item Generally, the improvements in the BOPE-B estimator over the BOPE estimator are driven by improvements in \emph{both} bias and variance, but in many cases, the improvement in variance is the dominant factor.    
    \item The SPPE method can perform quite poorly when the assumption on the apriori structure of demand does not hold.  
\end{itemize}



\subsubsection{Synthetic Datasets}
\label{sec:mse_synthetic}
We present results for two different demand functions. 

{\bf (a) A Simple Demand Function}

The features $\bx_i$ are generated uniformly random from the square $[-1,1]^2$. The logging pricing policy is $P_i=\frac{1}{2}\bx_i^\top [1,-1] + 7 + \epsilon_i$, where $\epsilon_i\sim\mathcal{N}(0,2)$ are i.i.d. noise. The target pricing policy is $P_i = \frac{1}{2}\bx_i^\top [1,-1] + b + \epsilon_i$, where $b$ is chosen from $\{2,3,4\}$ and then fixed throughout each experiment. We present results for each value of $b$. 

The demand function is 
\begin{align*}
d(\bx,p) &= \frac{1}{4} + \frac{3}{4}\sigma
\left(5 - \frac{1}{2} p - \bx^\top [-1,1]\right),\\
\text{where} \ \sigma(y) &= \frac{1}{1+e^{-y}}.
\end{align*}
 The sigmoid function $\sigma(y)$ is used to ensure \textit{(i)} demand is within [0,1] \textit{(ii)} demand decreases while price increases.

We fix the sample size to be $n=50$ throughout the experiment. We use the ground truth to simulate realizations of the binary demand vector corresponding to these 50 sample points. We repeat the procedure 100 times to obtain the bias, variance, and MSE of the four estimators. We perform the experiment for 30 different random seeds and report the average results in Table~\ref{table:synthetic_appendix_1}. Notice for each random seed, we sample a different set of features and prices.

\begin{table}[htb!]
\begin{tabular}{@{}lllll@{}}
\toprule
\textbf{Metrics} & \textbf{BOPE-B} & \textbf{BOPE} & \textbf{LASSO} & \textbf{SPPE} \\ \midrule
\multicolumn{5}{c}{Target Policy has $b=2$.} \\
\midrule
  MSE & 1.63 & 1.83 & 1.71 & 1.08\\
  Bias$^2$ & 0.22 & 0.27 & 0.25 & 0.17\\
  Variance & 1.41 & 1.56 & 1.46 & 0.91\\
\midrule
\multicolumn{5}{c}{Target Policy has $b=3$.} \\
\midrule
  MSE & 1.73 & 1.95 & 1.80 & 1.92\\
  Bias$^2$ & 0.33 & 0.40 & 0.35 & 0.17\\
  Variance & 1.40 & 1.55 & 1.45 & 1.75\\
\midrule
\multicolumn{5}{c}{Target Policy has $b=4$.} \\
\midrule
  MSE & 1.50 & 1.81 & 1.57 & 1.60\\
  Bias$^2$ & 0.31 & 0.38 & 0.33 & 0.16\\
  Variance & 1.19 & 1.43 & 1.24 & 1.44\\
  \bottomrule
\end{tabular}
\caption{Decomposition of the mean squared error. Synthetic dataset setting (a).}
\label{table:synthetic_appendix_1}
\end{table}

{\bf (b) A Different Demand Function}

We consider a different demand function 
\begin{align*}
d(\bx,p) &= \frac{1}{4} + \frac{3}{4}\sigma
\left(5 - \frac{1}{2} p - \arctan(\bx_1/\bx_2)\right).
\end{align*}
Notice this demand function is more complicated than that in setting (a). In the sigmoid function, we now have a nonlinear function $\arctan(\bx_1/\bx_2)$ instead of the linear function $\bx^\top [-1,1]$. 

The rest of the set up is the same as in part (a). We repeat the experiment for 30 different random seeds and report the average results in Table~\ref{table:synthetic_appendix_2}.

\begin{table}[htb!]
\begin{tabular}{@{}lllll@{}}
\toprule
\textbf{Metrics} & \textbf{BOPE-B} & \textbf{BOPE} & \textbf{LASSO} & \textbf{SPPE} \\ \midrule
\multicolumn{5}{c}{Target Policy has $b=2$.} \\
\midrule
  MSE & 1.18 & 1.47 & 1.21 & 1.43\\
  Bias$^2$ & 0.20 & 0.25 & 0.22 & 0.14\\
  Variance & 0.98 & 1.22 & 0.99 & 1.29\\
\midrule
\multicolumn{5}{c}{Target Policy has $b=3$.} \\
\midrule
  MSE & 2.09 & 2.30 & 2.13 & 2.22\\
  Bias$^2$ & 0.52 & 0.57 & 0.54 & 0.46\\
  Variance & 1.57 & 1.73 & 1.59 & 1.76\\
\midrule
\multicolumn{5}{c}{Target Policy has $b=4$.} \\
\midrule
  MSE &  1.99  & 2.18   & 2.05  & 2.42  \\
  Bias$^2$ &  0.38  &  0.45  & 0.41  & 0.27  \\
  Variance &  1.61 & 1.73 & 1.64 & 2.15\\
  \bottomrule
\end{tabular}
\caption{Decomposition of the mean squared error. Synthetic dataset setting (b).}
\label{table:synthetic_appendix_2}
\end{table}

\subsubsection{A Real World Dataset}
\label{sec:mse_nomis}
We conduct experiments on a real world dataset of auto loan applications collected by a major auto lender in North America. The dataset was first studied by \citet{phillips2015effectiveness} and later used to evaluate personalized pricing algorithms by \citet{ban2021personalized}. The dataset includes data collected over a period of several years. We present results for 5 different subsets of the Nomis dataset. To train the models, we use two covariates: FICO score and requested loan amount. We use the offered interest rate as price. We consider four target policies that take the original prices and increase/decrease them by 5 or 10\%.

 We impute counterfactuals, including the expected demand, using XGBoost trained on the entire subset to represent the ground truth model. We choose $n=50$ and sample these points randomly from the dataset. We use the ground truth to simulation 100 realizations of the demand vector corresponding to these 50 sample points, which we use to obtain the bias and variance of the different estimators. We repeat the experiment 30 times (with a different training set each time) and report the average results. 

 In Table~\ref{table:nomis_mse_1} and ~\ref{table:nomis_mse_2}, we present results obtained from 2 different subsets of the Nomis dataset. In Appendix \ref{sec:nomis}, we provide results obtained from 3 other subsets of the Nomis dataset. 
 
\begin{table}[htb!]
\begin{tabular}{@{}lllll@{}}
\toprule
\textbf{Metrics} & \textbf{BOPE-B} & \textbf{BOPE} & \textbf{LASSO} & \textbf{SPPE} \\ \midrule
\multicolumn{5}{c}{Target Policy is 5\% increase.} \\
\midrule
  MSE & 0.11 & 0.13 & 0.17 & 0.80\\
  Bias$^2$ & 0.03 & 0.04 & 0.06 & 0.34\\
  Variance & 0.08 & 0.09 & 0.11 & 0.46\\
\midrule
\multicolumn{5}{c}{Target Policy is 5\% decrease.} \\
\midrule
  MSE & 0.03 & 0.05 & 0.10 & 0.20\\
  Bias$^2$ & 0.01 & 0.02 & 0.05 & 0.10\\
  Variance & 0.02 & 0.03 & 0.05 & 0.10\\
\midrule
\multicolumn{5}{c}{Target Policy is 10\% increase.} \\
\midrule
  MSE & 0.37 & 0.41 & 0.44 & 1.23\\
  Bias$^2$ & 0.12 & 0.13 & 0.15 & 0.49\\
  Variance & 0.25 & 0.28 & 0.29 & 0.74\\
\midrule
\multicolumn{5}{c}{Target Policy is 10\% decrease.} \\
\midrule
  MSE & 0.009 & 0.007 & 0.018 & 0.034\\
  Bias$^2$ & 0.003 & 0.003 & 0.008 & 0.012\\
  Variance & 0.006 & 0.004 & 0.010 & 0.022\\
  \bottomrule
\end{tabular}
\caption{For each target policy and for each method, we present the MSE, bias squared, and variance. Results obtained from a subset of the Nomis dataset with Year = 2003, Tier = 1, Car Type = Used, Term = 60, and Partner Bin = 1. There are 1,065 datapoints in the subset.}
\label{table:nomis_mse_1}
\end{table}

\begin{table}[htb!]
\begin{tabular}{@{}lllll@{}}
\toprule
\textbf{Metrics} & \textbf{BOPE-B} & \textbf{BOPE} & \textbf{LASSO} & \textbf{SPPE} \\ \midrule
\multicolumn{5}{c}{Target Policy is 5\% increase.} \\
\midrule
  MSE & 0.39 & 0.46 & 0.45 & 0.80\\
  Bias$^2$ & 0.09 & 0.11 & 0.11 & 0.15\\
  Variance & 0.30 & 0.35 & 0.34 & 0.65\\
\midrule
\multicolumn{5}{c}{Target Policy is 5\% decrease.} \\
\midrule
  MSE & 0.69 & 0.75 & 0.77 & 0.58\\
  Bias$^2$ & 0.25 & 0.29 & 0.30 & 0.14\\
  Variance & 0.44 & 0.46 & 0.47 & 0.44\\
\midrule
\multicolumn{5}{c}{Target Policy is 10\% increase.} \\
\midrule
  MSE & 0.78 & 0.92 & 0.86 & 0.80\\
  Bias$^2$ & 0.26 & 0.30 & 0.30 & 0.13\\
  Variance & 0.52 & 0.62 & 0.56 & 0.67\\
\midrule
\multicolumn{5}{c}{Target Policy is 10\% decrease.} \\
\midrule
  MSE & 0.38 & 0.56 & 0.36 & 0.38\\
  Bias$^2$ & 0.10 & 0.10 & 0.10 & 0.10\\
  Variance & 0.28 & 0.46 & 0.26 & 0.28\\
  \bottomrule
\end{tabular}
\caption{For each target policy and for each method, we present the MSE, bias squared, and variance. Results obtained from a subset of the Nomis dataset with Year from 2002 to 2004, Tier = 3, Car Type = Used, Term = 48, and Partner Bin = 3. There are 578 datapoints in the subset.}
\label{table:nomis_mse_2}
\end{table}

\begin{table*}[htb!]
    \begin{center}
        \begin{tabular}{c@{}p{8mm}c cc@{} c cc@{}}
        \toprule
        \multirow{2}{*}{\textbf{Target Policy}} & \multirow{2}{*}{\ \ \textbf{$\mathcal R$}} & \multicolumn{2}{c}{\textbf{BOPE-Bern}} 
         & & \multicolumn{2}{c}{\textbf{BOPE-B}}\\
        \cmidrule{3-4} \cmidrule{6-7}
    & & $Bern(\bw^B,r^{WC}(\bw^B))$ 
        & $\hat{\mathcal R}(\bw^B)$
        & & $Bern(\bw^{MSE},r^{WC}(\bw^{MSE}))$ & $\hat{\mathcal R}(\bw^{MSE})$\\
        \midrule
        b = 2 & 3.51 & 1.07 $\pm$ 0.050 & 3.03 $\pm$ 0.054 & & 0.10 $\pm$ 0.016 & 3.08 $\pm$ 0.053\\
        b = 3 & 5.10 & 1.97 $\pm$ 0.098 & 4.51 $\pm$ 0.102 & & 0.10 $\pm$ 0.024 & 4.53 $\pm$ 0.099\\
        b = 4 & 4.95 & 1.81 $\pm$ 0.064 & 4.91 $\pm$ 0.048 & & 0.04 $\pm$ 0.014 & 4.75 $\pm$ 0.054\\
        \bottomrule
        \end{tabular}
        \caption{We present average and standard error of revenue bounds, computed from 100 demand realizations. The bounds in \textit{BOPE-B} are the worst-case Bernstein bounds with \textit{BOPE-B} weights. Results obtained from synthetic dataset (a) described in Section~\ref{sec:mse_synthetic}.}
        \label{table:bernstein_1}
    \end{center}
\end{table*}

\begin{table*}[htb!]
    \begin{center}
        \begin{tabular}{c@{}p{8mm}c cc@{} c cc@{}}
        \toprule
        \multirow{2}{*}{\textbf{Target Policy}} & \multirow{2}{*}{\ \ \textbf{$\mathcal R$}} & \multicolumn{2}{c}{\textbf{BOPE-Bern}} 
         & & \multicolumn{2}{c}{\textbf{BOPE-B}}\\
        \cmidrule{3-4} \cmidrule{6-7}
    & & $Bern(\bw^B,r^{WC}(\bw^B))$ 
        & $\hat{\mathcal R}(\bw^B)$
        & & $Bern(\bw^{MSE},r^{WC}(\bw^{MSE}))$ & $\hat{\mathcal R}(\bw^{MSE})$\\
        \midrule
        b = 2 & 3.86 & 1.29 $\pm$ 0.051 & 3.29 $\pm$ 0.054 & & 0.08 $\pm$ 0.017 & 3.36 $\pm$ 0.054\\
        b = 3 & 4.59 & 1.76 $\pm$ 0.082 & 4.30 $\pm$ 0.073 & & 0.05 $\pm$ 0.021 & 4.29 $\pm$ 0.075\\
        b = 4 & 5.23 & 1.57 $\pm$ 0.081 & 4.86 $\pm$ 0.072 & & 0.01 $\pm$ 0.008 & 4.85 $\pm$ 0.075\\
        \bottomrule
        \end{tabular}
        \caption{We present average and standard error of revenue bounds, computed from 100 demand realizations. The bounds in \textit{BOPE-B} are the worst-case Bernstein bounds with \textit{BOPE-B} weights. Results obtained from synthetic dataset (b) described in Section~\ref{sec:mse_synthetic}.}
        \label{table:bernstein_2}
    \end{center}
\end{table*}

\begin{table*}[htb!]
    \begin{center}
        \begin{tabular}{c@{}p{8mm}c cc@{} c cc@{}}
        \toprule
        \multirow{2}{*}{\textbf{Target Policy}} & \multirow{2}{*}{\ \ \textbf{$\mathcal R$}} & \multicolumn{2}{c}{\textbf{BOPE-Bern}} 
         & & \multicolumn{2}{c}{\textbf{BOPE-B}}\\
        \cmidrule{3-4} \cmidrule{6-7}
    & & $Bern(\bw^B,r^{WC}(\bw^B))$ 
        & $\hat{\mathcal R}(\bw^B)$
        & & $Bern(\bw^{MSE},r^{WC}(\bw^{MSE}))$ & $\hat{\mathcal R}(\bw^{MSE})$\\
        \midrule
        +5\% & 4.22 & 1.99 $\pm$ 0.020 & 4.13 $\pm$ 0.017& & 1.52 $\pm$ 0.031 & 4.24 $\pm$ 0.014\\
        -5\% & 4.16 & 1.72 $\pm$ 0.007 & 3.87 $\pm$ 0.005 & & 1.17 $\pm$ 0.009 & 4.05 $\pm$ 0.004\\
        +10\% & 3.85 & 1.55 $\pm$ 0.034 & 3.70 $\pm$ 0.031 & & 1.06 $\pm$ 0.043 & 3.84 $\pm$ 0.028\\
        -10\% & 4.05 & 1.97 $\pm$ 0.003 & 3.82 $\pm$ 0.004 & & 1.46 $\pm$ 0.005 & 4.08 $\pm$ 0.003\\   
        \bottomrule
        \end{tabular}
        \caption{We present average and standard error of revenue bounds, computed from 100 demand realizations. The bounds in \textit{BOPE-B} are the worst-case Bernstein bounds with \textit{BOPE-B} weights. Results obtained from a subset of the Nomis dataset with year = 2003, Tier = 1, Car Type = Used, Term = 60, and Partner Bin = 1. There are 1,065 datapoints in the subset. }\label{table:bernstein_3}
    \end{center}
\end{table*}

\begin{table*}[htb!]
    \begin{center}
        \begin{tabular}{c@{}p{8mm}c cc@{} c cc@{}}
        \toprule
        \multirow{2}{*}{\textbf{Target Policy}} & \multirow{2}{*}{\ \ \textbf{$\mathcal R$}} & \multicolumn{2}{c}{\textbf{BOPE-Bern}} 
         & & \multicolumn{2}{c}{\textbf{BOPE-B}}\\
        \cmidrule{3-4} \cmidrule{6-7}
    & & $Bern(\bw^B,r^{WC}(\bw^B))$ 
        & $\hat{\mathcal R}(\bw^B)$
        & & $Bern(\bw^{MSE},r^{WC}(\bw^{MSE}))$ & $\hat{\mathcal R}(\bw^{MSE})$\\
        \midrule
        +5\% & 2.77 & 0.45 $\pm$ 0.021 & 2.42 $\pm$ 0.031 & & 0.00 $\pm$ 0.000 & 2.52 $\pm$ 0.032\\
        -5\% & 3.35 & 0.53 $\pm$ 0.028 & 2.64 $\pm$ 0.033 & & 0.00 $\pm$ 0.000 & 2.75 $\pm$ 0.030\\
        +10\% & 2.87 & 0.21 $\pm$ 0.018 & 2.45 $\pm$ 0.042 & & 0.00 $\pm$ 0.000 & 2.48 $\pm$ 0.045\\
        -10\% & 3.67 & 0.55 $\pm$ 0.027 & 2.81 $\pm$ 0.034 & & 0.00 $\pm$ 0.000 & 2.93 $\pm$ 0.034\\    
        \bottomrule
        \end{tabular}
        \caption{We present average and standard error of revenue bounds, computed from 100 demand realizations. The bounds in \textit{BOPE-B} are the worst-case Bernstein bounds with \textit{BOPE-B} weights. Results obtained from a subset of the Nomis dataset with year from 2002 to 2004, Tier = 3, Car Type = Used, Term = 48, and Partner Bin = 3. There are 578 datapoints in the subset. }\label{table:bernstein_4}
    \end{center}
\end{table*}

\subsection{Bernstein Bounds}
\label{sec:exps_bernstein}
We next consider $\bw^B$ and the corresponding estimator $\hat{\mathcal R}(\bw^B)$. We denote the corresponding method BOPE-Bern. Since the primary motivation of BOPE-Bern was to provide high-quality safe guarantees on the revenue, we focus our experiments on such safe guarantees, and specifically comparisons to BOPE.

Recall \cref{lem:rev_bounds} provides a safe guarantee for \emph{any} set of honest weights.  Hence, to form a safe guarantee for BOPE, we take the weights computed by BOPE, and then solve the inner maximization problem in \eqref{eq:genertic_obj} with the Bernstein bound objective for those weights.  Since the revenue must be non-negative, we take the positive part of the optimal value.  If weights computed by BOPE were honest, this procedure would yield a theoretically valid safe guarantee. Insofar as we specify hyperparameters in BOPE in a ``dishonest" fashion, the resulting safe guarantee is only heuristically valid.  (The same criticism holds for our own method, BOPE-B, making it a fair comparison.)  


Our experiments suggest BOPE-B yields much better safe guarantees than BOPE, while providing comparably good estimates of the actual revenue.
    
 
In Tables ~\ref{table:bernstein_1} and ~\ref{table:bernstein_2}, we present results on the two synthetic datasets described in Section~\ref{sec:mse_synthetic}. In Tables ~\ref{table:bernstein_3} and ~\ref{table:bernstein_4}, we present results on subsets of the Nomis dataset. The experiment details are the same as described in Section~\ref{sec:exps}. For each method, we present the one-sided 90\% confidence lower bound on revenue (i.e. we choose $\epsilon=0.1$). For all experiments in this subsection, we use sample size $n=50$.

\section{HYPER-PARAMETER HEURISTICS}
\label{sec:heuristics}
Our heuristics for fitting hyper-parameters are inspired by the heuristics of \cite{kallus2018balanced} for BOPE. 

Define the revenue random variable $R_i := p_i D_i$. 
Loosely, \cite{kallus2018balanced}  assumes that the $R_i$ are homoscedastic with variance $\sigma^2$ and mean $r(p_i, \bx_i)$ for each $i \in [n]$.  They then compute the worst-case MSE of the weighted doubly robust estimator over a suitable RKHS ball.  It turns out the resulting expression is identical to 
the expected MSE of this same estimator assuming the unknown expected revenue function was drawn from the following Gaussian Process Prior:
\begin{equation} \label{GP}
r(\cdot, \cdot)\sim \mathcal{GP}(\hat\br(\cdot, \cdot), \hat\Gamma^2 K(\cdot, \cdot)).
\end{equation}
Said differently, the worst-case MSE is equal to an expected MSE under a suitable prior.

Thus, \cite{kallus2018balanced} proposes to fit any hyperparameters needed for BOPE by using standard marginal likelihood techniques \citep[Chapt. 5]{williams2006gaussian} to instead fit the above Gaussian Process prior and then ``read off" the parameters needed for BOPE. 

We follow this same strategy in our experiments. For the kernel, we adopt a Gaussian kernel but standardize each component by its variance.  Specifically, we take 
\[
K(\bm z, \overline{\bm z}) := \exp\left(-(\bm z- \overline{\bm z})^\top \bm \Sigma^{-1} (\bm z - \overline{\bm z}) \right), 
\]
where $\bm z = (p, \bx) \in \R_+ \times \mathcal X$ and $\bm \Sigma$ is a diagonal matrix.  

We then optimize the choice of $\bm \Sigma$, $\sigma^2$ and $\hat \Gamma^2$ to  maximize the marginal likelihood of the data under the prior \cref{GP} assuming the likelihood $R_i\mid \bm r(\cdot, \cdot) \sim \mathcal N(r(p_i, \bx_i), \sigma^2)$.
Because the Gaussian process prior and Gaussian likelihood are conjugate, the resulting marginal likelihood has a nice closed-form expression and the entire optimization can be represented tractably.  (Again, see \cite{williams2006gaussian} for details.)

Unfortunately, for the case of BOPE-B, our expression for the worst-case MSE does not seem to match the expected MSE under a simple prior.  Hence, we heuristically seek parameters that maximize the marginal likelihood of the data under the model \cref{GP}, but now assuming that $D_i | P_i = p_i \sim \text{Bernoulli}(d(p_i, \bx_i))$ and $R_i = p_i D_i$.  In other words, we adjust the previous heuristic to account for the binary nature of demand.  For this binary likelihood, we do not have conjugacy, and so there is no simple closed-form expression for the marginal likelihood.  Instead, we follow \cite{flaxman2015fast} and employ a Laplace approximation to the marginal likelihood.  The resulting approximate likelihood does admit a simple form and the resulting maximal marginal likelihood optimization is tractable.  

For our BOPE-B method, we optimize this approximate marginal likelihood to fit \cref{GP}, and read off the necessary hyper-parameters.

\section{CONCLUSION}
\label{sec:conclusion}
In this paper, we have proposed a new approach for policy evaluation tailored to pricing applications. Our approaching uses special structures of pricing problems, including: \textit{(i)} demand observations are binary; \textit{(ii)} revenue per customer is nonnegative and no greater than the price offered; \textit{(iii)} revenue equals demand times price; \textit{(iv)} the value of the target policy can be very different from that of the logging policy, and thus weights do not need to sum to $n$. 
We compute weights to optimize either \textit{(i)} the worst-case mean squared error of our estimate or \textit{(ii)} a worst-case lower bound on the unknown revenue of the target policy. In both cases, the worst-case is taken over a set of plausible revenue functions described by an RKHS ball. We establish theoretical guarantees showing our weighted revenue estimator converges under overlap assumptions and empirically demonstrate the advantage of our approach using a real-world pricing dataset where there is little overlap. Future work  
might consider specialized algorithms for computing the weights in our method given its special structure, 
e.g., adapting the Mirror Prox algorithm of \citep{nemirovski2004prox}, the primal-dual method in \citep{nesterov2007dual}, or various algorithms for saddle point problems \citep{juditsky2011solving,mertikopoulos2019optimistic}. 

\subsubsection*{Acknowledgements}
The authors are listed in alphabetical order. We acknowledge the support of NSF grants CMMI-1763000, CMMI-1944428, and IIS-2147361

\bibliography{references}

\begin{thebibliography}{}

\bibitem[Aouad et~al., 2019]{aouad2019market}
Aouad, A., Elmachtoub, A.~N., Ferreira, K.~J., and McNellis, R. (2019).
\newblock Market segmentation trees.
\newblock {\em arXiv preprint arXiv:1906.01174}.

\bibitem[Athey and Wager, 2021]{athey2021policy}
Athey, S. and Wager, S. (2021).
\newblock Policy learning with observational data.
\newblock {\em Econometrica}, 89(1):133--161.

\bibitem[Baardman et~al., 2019]{baardman2019scheduling}
Baardman, L., Cohen, M.~C., Panchamgam, K., Perakis, G., and Segev, D. (2019).
\newblock Scheduling promotion vehicles to boost profits.
\newblock {\em Management Science}, 65(1):50--70.

\bibitem[Ban and Keskin, 2021]{ban2021personalized}
Ban, G.-Y. and Keskin, N.~B. (2021).
\newblock Personalized dynamic pricing with machine learning: High-dimensional
  features and heterogeneous elasticity.
\newblock {\em Management Science}, 67(9):5549--5568.

\bibitem[Bertsekas, 1997]{bertsekas1997nonlinear}
Bertsekas, D.~P. (1997).
\newblock Nonlinear programming.
\newblock {\em Journal of the Operational Research Society}, 48(3):334--334.

\bibitem[Besbes et~al., 2010]{besbes2010testing}
Besbes, O., Phillips, R., and Zeevi, A. (2010).
\newblock Testing the validity of a demand model: An operations perspective.
\newblock {\em Manufacturing \& Service Operations Management}, 12(1):162--183.

\bibitem[Biggs, 2022]{biggs2022convex}
Biggs, M. (2022).
\newblock Convex loss functions for contextual pricing with observational
  posted-price data.
\newblock {\em arXiv preprint arXiv:2202.10944}.

\bibitem[Biggs et~al., 2021]{biggs2021loss}
Biggs, M., Gao, R., and Sun, W. (2021).
\newblock Loss functions for discrete contextual pricing with observational
  data.
\newblock {\em arXiv preprint arXiv:2111.09933}.

\bibitem[Bottou et~al., 2013]{bottou2013counterfactual}
Bottou, L., Peters, J., Qui{\~n}onero-Candela, J., Charles, D.~X., Chickering,
  D.~M., Portugaly, E., Ray, D., Simard, P., and Snelson, E. (2013).
\newblock Counterfactual reasoning and learning systems: The example of
  computational advertising.
\newblock {\em Journal of Machine Learning Research}, 14(11).

\bibitem[Boucheron et~al., 2013]{boucheron2013concentration}
Boucheron, S., Lugosi, G., and Massart, P. (2013).
\newblock {\em Concentration inequalities: A nonasymptotic theory of
  independence}.
\newblock Oxford university press.

\bibitem[Bu et~al., 2022]{bu2022offline}
Bu, J., Simchi-Levi, D., and Wang, L. (2022).
\newblock Offline pricing and demand learning with censored data.
\newblock {\em Management Science}.

\bibitem[Cai et~al., 2021]{cai2021deep}
Cai, H., Shi, C., Song, R., and Lu, W. (2021).
\newblock Deep jump learning for off-policy evaluation in continuous treatment
  settings.
\newblock {\em Advances in Neural Information Processing Systems},
  34:15285--15300.

\bibitem[Chen et~al., 2022]{chen2022statistical}
Chen, X., Owen, Z., Pixton, C., and Simchi-Levi, D. (2022).
\newblock A statistical learning approach to personalization in revenue
  management.
\newblock {\em Management Science}, 68(3):1923--1937.

\bibitem[Chernozhukov et~al., 2019]{chernozhukov2019semi}
Chernozhukov, V., Demirer, M., Lewis, G., and Syrgkanis, V. (2019).
\newblock Semi-parametric efficient policy learning with continuous actions.
\newblock {\em Advances in Neural Information Processing Systems}, 32.

\bibitem[Cohen et~al., 2022]{cohen2022price}
Cohen, M.~C., Elmachtoub, A.~N., and Lei, X. (2022).
\newblock Price discrimination with fairness constraints.
\newblock {\em Management Science}, 68(12):8536--8552.

\bibitem[Cohen et~al., 2017]{cohen2017impact}
Cohen, M.~C., Leung, N.-H.~Z., Panchamgam, K., Perakis, G., and Smith, A.
  (2017).
\newblock The impact of linear optimization on promotion planning.
\newblock {\em Operations Research}, 65(2):446--468.

\bibitem[Dud\'{\i}k et~al., 2011]{dudik2011doubly}
Dud\'{\i}k, M., Langford, J., and Li, L. (2011).
\newblock Doubly robust policy evaluation and learning.
\newblock In {\em International Conference on Machine Learning}, page
  1097–1104. PMLR.

\bibitem[Elliott, 2008]{elliott2008model}
Elliott, M.~R. (2008).
\newblock Model averaging methods for weight trimming.
\newblock {\em Journal of official statistics}, 24(4):517.

\bibitem[Elmachtoub et~al., 2021]{elmachtoub2021value}
Elmachtoub, A.~N., Gupta, V., and Hamilton, M.~L. (2021).
\newblock The value of personalized pricing.
\newblock {\em Management Science}, 67(10):6055--6070.

\bibitem[Ferreira et~al., 2016]{ferreira2016analytics}
Ferreira, K.~J., Lee, B. H.~A., and Simchi-Levi, D. (2016).
\newblock Analytics for an online retailer: Demand forecasting and price
  optimization.
\newblock {\em Manufacturing \& service operations management}, 18(1):69--88.

\bibitem[Flaxman et~al., 2015]{flaxman2015fast}
Flaxman, S., Wilson, A., Neill, D., Nickisch, H., and Smola, A. (2015).
\newblock Fast kronecker inference in gaussian processes with non-gaussian
  likelihoods.
\newblock In {\em International Conference on Machine Learning}, pages
  607--616. PMLR.

\bibitem[Hanna et~al., 2017]{hanna2017bootstrapping}
Hanna, J.~P., Stone, P., and Niekum, S. (2017).
\newblock Bootstrapping with models: Confidence intervals for off-policy
  evaluation.
\newblock In {\em Thirty-First AAAI Conference on Artificial Intelligence}.

\bibitem[Ionides, 2008]{ionides2008truncated}
Ionides, E.~L. (2008).
\newblock Truncated importance sampling.
\newblock {\em Journal of Computational and Graphical Statistics},
  17(2):295--311.

\bibitem[Juditsky et~al., 2011]{juditsky2011solving}
Juditsky, A., Nemirovski, A., and Tauvel, C. (2011).
\newblock Solving variational inequalities with stochastic mirror-prox
  algorithm.
\newblock {\em Stochastic Systems}, 1(1):17--58.

\bibitem[Kallus, 2018]{kallus2018balanced}
Kallus, N. (2018).
\newblock Balanced policy evaluation and learning.
\newblock {\em Advances in neural information processing systems}, 31.

\bibitem[Kallus, 2020]{kallus2020comment}
Kallus, N. (2020).
\newblock Comment: Entropy learning for dynamic treatment regimes.
\newblock {\em arXiv preprint arXiv:2004.02778}.

\bibitem[Kallus, 2021]{kallus2021more}
Kallus, N. (2021).
\newblock More efficient policy learning via optimal retargeting.
\newblock {\em Journal of the American Statistical Association},
  116(534):646--658.

\bibitem[Kallus and Zhou, 2018]{kallus2018policy}
Kallus, N. and Zhou, A. (2018).
\newblock Policy evaluation and optimization with continuous treatments.
\newblock In {\em International conference on artificial intelligence and
  statistics}, pages 1243--1251. PMLR.

\bibitem[Kolbeinsson et~al., 2022]{kolbeinsson2022galactic}
Kolbeinsson, A., Shukla, N., Gupta, A., Marla, L., and Yellepeddi, K. (2022).
\newblock Galactic air improves ancillary revenues with dynamic personalized
  pricing.
\newblock {\em INFORMS Journal on Applied Analytics}.

\bibitem[Mertikopoulos et~al., 2019]{mertikopoulos2019optimistic}
Mertikopoulos, P., Lecouat, B., Zenati, H., Foo, C.-S., Chandrasekhar, V., and
  Piliouras, G. (2019).
\newblock Optimistic mirror descent in saddle-point problems: Going the extra
  (gradient) mile.
\newblock In {\em ICLR 2019-7th International Conference on Learning
  Representations}, pages 1--23.

\bibitem[Nambiar et~al., 2019]{nambiar2019dynamic}
Nambiar, M., Simchi-Levi, D., and Wang, H. (2019).
\newblock Dynamic learning and pricing with model misspecification.
\newblock {\em Management Science}, 65(11):4980--5000.

\bibitem[Nemirovski, 2004]{nemirovski2004prox}
Nemirovski, A. (2004).
\newblock Prox-method with rate of convergence o (1/t) for variational
  inequalities with lipschitz continuous monotone operators and smooth
  convex-concave saddle point problems.
\newblock {\em SIAM Journal on Optimization}, 15(1):229--251.

\bibitem[Nesterov, 2007]{nesterov2007dual}
Nesterov, Y. (2007).
\newblock Dual extrapolation and its applications to solving variational
  inequalities and related problems.
\newblock {\em Mathematical Programming}, 109(2):319--344.

\bibitem[Phillips et~al., 2015]{phillips2015effectiveness}
Phillips, R., {\c{S}}im{\c{s}}ek, A.~S., and Van~Ryzin, G. (2015).
\newblock The effectiveness of field price discretion: Empirical evidence from
  auto lending.
\newblock {\em Management Science}, 61(8):1741--1759.

\bibitem[Qi et~al., 2022]{qi2022offline}
Qi, Z., Tang, J., Fang, E., and Shi, C. (2022).
\newblock Offline personalized pricing with censored demand.
\newblock {\em Available at SSRN}.

\bibitem[Sachdeva et~al., 2020]{sachdeva2020off}
Sachdeva, N., Su, Y., and Joachims, T. (2020).
\newblock Off-policy bandits with deficient support.
\newblock In {\em Proceedings of the 26th ACM SIGKDD International Conference
  on Knowledge Discovery \& Data Mining}, pages 965--975.

\bibitem[Smola and Sch{\"o}lkopf, 2004]{smola2004tutorial}
Smola, A.~J. and Sch{\"o}lkopf, B. (2004).
\newblock A tutorial on support vector regression.
\newblock {\em Statistics and computing}, 14(3):199--222.

\bibitem[Sondhi et~al., 2020]{sondhi2020balanced}
Sondhi, A., Arbour, D., and Dimmery, D. (2020).
\newblock Balanced off-policy evaluation in general action spaces.
\newblock In {\em International Conference on Artificial Intelligence and
  Statistics}, pages 2413--2423. PMLR.

\bibitem[Swaminathan and Joachims, 2015a]{swaminathan2015counterfactual}
Swaminathan, A. and Joachims, T. (2015a).
\newblock Counterfactual risk minimization: Learning from logged bandit
  feedback.
\newblock In {\em International Conference on Machine Learning}, pages
  814--823. PMLR.

\bibitem[Swaminathan and Joachims, 2015b]{swaminathan2015self}
Swaminathan, A. and Joachims, T. (2015b).
\newblock The self-normalized estimator for counterfactual learning.
\newblock {\em advances in neural information processing systems}, 28.

\bibitem[Thomas and Brunskill, 2016]{thomas2016data}
Thomas, P. and Brunskill, E. (2016).
\newblock Data-efficient off-policy policy evaluation for reinforcement
  learning.
\newblock In {\em International Conference on Machine Learning}, pages
  2139--2148. PMLR.

\bibitem[Wahba, 1990]{wahba1990spline}
Wahba, G. (1990).
\newblock {\em Spline models for observational data}.
\newblock SIAM.

\bibitem[Wang et~al., 2021]{wang2021uncertainty}
Wang, Y., Chen, X., Chang, X., and Ge, D. (2021).
\newblock Uncertainty quantification for demand prediction in contextual
  dynamic pricing.
\newblock {\em Production and Operations Management}, 30(6):1703--1717.

\bibitem[Wang and Zheng, 2021]{wang2021measuring}
Wang, Y. and Zheng, Z. (2021).
\newblock Measuring policy performance in online pricing with offline data.
\newblock {\em Available at SSRN 3729003}.

\bibitem[Wang et~al., 2017]{wang2017optimal}
Wang, Y.-X., Agarwal, A., and Dud{\i}k, M. (2017).
\newblock Optimal and adaptive off-policy evaluation in contextual bandits.
\newblock In {\em International Conference on Machine Learning}, pages
  3589--3597. PMLR.

\bibitem[Williams and Rasmussen, 2006]{williams2006gaussian}
Williams, C.~K. and Rasmussen, C.~E. (2006).
\newblock {\em Gaussian processes for machine learning}, volume~2.
\newblock MIT press Cambridge, MA.

\end{thebibliography}


\appendix
\onecolumn
\section{Proof of Properties of Weighted Revenue Estimators} \label{sec:proofs2}
\label{sec:properties_proof}

\begin{proof}[Proof of Lemma~\ref{lem:BiasVar}]
By definition of $\mathcal R$ and $r_j$, we have that $\mathcal R = \frac{1}{n}\sum_{i=1}^n r_{n+i}$. Since $\Eb{p_i D_i} = r_i$, it follows from the definitions of $\hat{\mathcal R}(\bw)$ that  $ Bias(\bw, \br)=\Eb{\hat{\mathcal R}(\bw) - \mathcal R}  = \bb(\bw)^\top \left(\br - \hat\br\right) $.


For the variance, we see that 
\begin{align*}
 Var(\bw, \br)& = \text{Var}(\hat{\mathcal R}(\bw))  = 
\text{Var}\left(\frac{1}{n}\sum_{i=1}^n w_i p_i D_i \right) 
=
\frac{1}{n^2}\sum_{i=1}^n w_i^2 p_i^2 \text{Var}\left(D_i\right)
\\
&=
\frac{1}{n^2}\sum_{i=1}^n w_i^2 p_i^2 d(\bx_i,p_i) (1-d(\bx_i,p_i)) 
=
\frac{1}{n^2}\sum_{i=1}^n w_i^2 r_i (p_i - r_i)
=
\frac{1}{n^2}\sum_{i=1}^n w_i^2 p_i r_i - \frac{1}{n^2}\sum_{i=1}^n w_i^2 (r_i)^2
	\\ &=\bv(\bw)^\top\br - 
	\frac{1}{n^2}\br^\top \begin{pmatrix} \text{diag}(w_1^2, \ldots, w_n^2) & \bm 0 
	\\ \bm 0 & \bm 0
			 \end{pmatrix} \br .
\end{align*}

The expression for MSE follows from the usual bias-variance decomposition.
\end{proof}

\begin{proof}[Proof of Lemma~\ref{lem:rev_bounds}]
Write
$$
\hat{\mathcal{R}}(\bw)-\mathcal{R}=\boldsymbol{b}(\bw)^{\top}\left(\br-\hat\br\right)+\sum_{i=1}^n w_i\left(p_i D_i - r_i\right)
$$
The first term is the bias of our estimator, evaluated in Lemma~\ref{lem:BiasVar}. The second term is a sum of mean-zero independent random variables. 
%
From \citep[Thm. 2.10]{boucheron2013concentration} and surrounding discussion (i.e. Bernstein's inequality), we have that with probability at least $1-\epsilon$, 
\begin{align*}
	&\frac{1}{n}\sum_{i=1}^n w_i(p_i D_i -  r_i) 
	\leq  \sqrt{2 Var(\bw; \br)\log(1/\epsilon)} - \frac{\max_{1 \leq i \leq n} |w_i| p_i\log(1/\epsilon)}{3n}.
\end{align*}
Combining completes the proof.
\end{proof}

\section{Reformulation of the Worst-Case Bernstein Inner Problem} \label{sec:reformulation_bern}

For the Bernstein objective, the inner maximization problem can be reformulated as follows:
$$
\begin{aligned}
\br^{WC}(\bw) \in \underset{\br, t}{\operatorname{argmax}} \quad & \ \bw^{\top}\left(\br-\hat\br\right)+\sqrt{2 \log (1 / \epsilon)} \cdot t \\
\text {s.t.} \quad & 0 \leq \br \leq \bp \\
& t^2 \leq \boldsymbol{v}(\bw)^{\top} \br-\br^{\top} \boldsymbol{Q}(\bw) \br,
\\
& \left(\br-\hat\br\right)^{\top} \boldsymbol{G}^{-1}\left(\br-\hat\br\right) \ \leq\ \Gamma^2,
\end{aligned}
$$
where $\boldsymbol{Q}( \bw):=\operatorname{diag}\left(\bw_1^2, \ldots, \bw_n^2, 0, \ldots, 0\right) \in \mathbb{R}^{2 n \times 2 n}$.
\section{Proof of \cref{thm:mse_asymptotics} }\label{sec:mseproof}

Recall the following classical fact about inverse propensity score weights:
\begin{lemma} \label{IPWResult}
For any function $f: \R \mapsto \R$ and any $i = 1, \ldots,n$ such that the expectations exist, we have the following identity:
\[
	\Eb{ W_i^{IP} f(P_i)} = \Eb{f(P_{n+i})}.
\]
\end{lemma}
\begin{proof}
Simply write the integrals:
\[
	\Eb{ W_i^{IP} f(P_i)} 
	\ = \ 
	\int_{p \in \R} f(p) \frac{g_1(p, \bx_i)}{g_0(p, \bx_i)} \ g_0(p, \bx_i) dp
	\ = \ 
	\int_{p \in \R} f(p) g_1(p, \bx_i)dp 
	\ = \ 
	\Eb{f(P_{n+i})}.
\]
\end{proof}
In particular, the lemma implies that 
$	\Eb{\mathcal R(\bw^{IP}/n)} = \Eb{\mathcal R},
$
i.e., using the (scaled) IP weights yields an unbiased estimator. 

Finally, for convenience, define the function 
\begin{align*}
\text{WCMSE}(\bw; \hat \Gamma)  \ := \ \max_{\br} \quad & 
\bDelta^\top \bb(\bw) \bb(\bw)^\top \bDelta + \text{Var}(\bw; \hat \br + \bDelta) .
\\
\text{s.t.} \quad & \bm 0 \leq \br \leq \bpi ,\ \ \ \ (\br - \hat \br)^\top \bG^{-1}(\br - \hat \br) \leq \hat \Gamma^2.
\end{align*}

A challenge in our analysis that $\hat{\Gamma}$ might be misspecified, i.e., it might be much smaller than $\Gamma$.  
Hence, $\text{WCMSE}(\bw; \hat{\Gamma})$ may not upper bound $MSE(\bw, \br)$.  

The next lemma shows we can cover such misspecification by inflating the worst-case MSE by a constant.  
\begin{lemma} \label{lem:RelatingMSEToWCMSE}
For any $\bw$, 
\(
\text{MSE}(\bw, \br) \ \leq \ \max\left(1, \frac{\Gamma^2}{\hat \Gamma^2}\right) \cdot \text{WCMSE}(\bw; \hat \Gamma).
\)
\end{lemma}
\begin{proof}
If $\Gamma \leq \hat \Gamma$, then the unknown revenue function $\br = \hat \br + \bDelta$ is feasible in the inner maximization defining $\bw^{MSE}(\hat \Gamma)$, so that $\text{MSE}(\bw^{MSE}(\hat \Gamma); \hat \br + \bDelta) \leq \text{WCMSE}(\bw^{MSE}(\hat \Gamma); \hat \Gamma)$.  
We thus focus on the case when $\Gamma > \hat \Gamma$ .  Then, 
\begin{align*}
\text{MSE}(\bw, \hat \br + \bDelta) 
& \ = \ 
\bDelta^\top \bb(\bw) \bb(\bw)^\top \bDelta + \text{Var}(\bw; \hat \br + \bDelta)
\\
& \ = \ 
\frac{\Gamma^2}{\hat \Gamma^2} \left( 
\frac{\hat \Gamma }{\Gamma}\bDelta^\top \bb(\bw) \bb(\bw)^\top \bDelta\frac{\hat \Gamma }{\Gamma}\right) +
  \text{Var}(\bw; \hat \br + \bDelta).
\end{align*}

Now consider the variance term.  From the proof of \cref{lem:BiasVar},
\begin{align*}
\text{Var}(\bw, \hat \br + \bDelta) 
& \ = \ 
\frac{1}{n^2}\sum_{i=1}^n w_i^2 (\hat \br_i + \Delta_i)(p_i - \hat \br_i - \Delta_i). 
\end{align*}
Since $\hat \br_i\geq 0$, and $\hat \br_i \leq p_i$, 
\[
	\hat \br_i + \Delta_i \leq \frac{\Gamma}{\hat \Gamma}\left(\hat \br_i + \frac{\hat \Gamma}{\Gamma}\Delta_i\right), 
\quad 
\text{ and } 
\quad
	p_i - \hat \br_i - \Delta_i \ \leq \ 
	\frac{\Gamma}{\hat \Gamma}\left(p_i - \hat \br_i - \frac{\hat \Gamma}{\Gamma}\Delta_i\right).
\]
Substituting above shows that 
\(
	\text{Var}(\bw, \hat \br + \bDelta) \ \leq \ \frac{\Gamma^2}{\hat \Gamma^2} \text{Var}(\bw, \hat \br + \frac{\hat \Gamma}{\Gamma}\bDelta).
\)
In summary, we have shown that 
\[
	\text{MSE}(\bw; \hat \br + \bDelta) \ \leq \frac{\Gamma^2}{\hat \Gamma^2} \text{MSE}(\bw; \hat \br + \frac{\hat \Gamma}{\Gamma}\bDelta).
\]
To complete the proof, note that $\hat \br + \frac{\hat \Gamma}{\Gamma}\bDelta$ is feasible in the optimization defining $\text{WCMSE}(\bw; \hat\Gamma)$. 
\end{proof}


\begin{proof}[Proof of Theorem~\ref{thm:mse_asymptotics}]
From \cref{lem:RelatingMSEToWCMSE}, it suffices to show that $\text{WCMSE}(\bw^{MSE}(\hat \Gamma); \hat \Gamma) = O_p(1/ n)$.  We show this latter claim by relating $\bw^{MSE}(\hat \Gamma)$ with the scaled inverse propensity weights $\bm W^{IP}/n$.

Specifically, since $\bm W^{IP}/n$ is feasible in the outer optimization problem defining $\bw^{MSE}(\hat \Gamma)$ we have that
\begin{align*}
 &\text{WCMSE}(\bw^{MSE}(\hat \Gamma); \hat \Gamma) 
 \\ \qquad & \ \leq \ 
 \text{WCMSE}(\bm W^{IP}/n; \hat \Gamma)
 \\
& \ \leq \ 
\max_{\br : (\br - \hat \br)^\top \bG^{-1}(\br - \hat \br) \leq \hat \Gamma^2} \text{MSE}(\bm W^{IP}/n; \br)
\\
& \ \leq \ 
\max_{\br : (\br - \hat \br)^\top \bG^{-1}(\br - \hat \br) \leq \hat \Gamma^2}  \ 
(\br - \hat \br)^\top \bb(\bm W^{IP}/n) \bb(\bm W^{IP}/n)^\top(\br - \hat \br) 
+ 
\frac{1}{4 n^2}\sum_{i=1}^n (W_i^{IP} P_i)^2 ,
\end{align*} 
where the second to last inequality follows by expanding the feasible region and the last by upper bounding the variance since $d_i(1-d_i) \leq \frac{1}{4}$.  We evaluate the maximization in closed form and round the constants up to $1$ yielding

\begin{equation} \label{eq:TargetExpansion_2}
	\text{WCMSE}(\bw^{MSE}(\hat \Gamma); \hat \Gamma) \ \leq \ 
	\hat \Gamma^2 \bb(\bm W^{IP})^\top \bG \bb(\bm W^{IP}) + \frac{1}{n^2}\sum_{i=1}^n (W_i^{IP}P_i)^2.
\end{equation}

We tackle the first term by upper bounding its expectation and applying Markov's inequality.  Using the definition of $\bG$ and $\bb(\bm W^{IP})$, write
\begin{align*}
\Eb{\bb(\bm W^{IP})^\top \bG \bb(\bm W^{IP})}
& \ = \ 
\frac{1}{n^2}\sum_{i=1}^n \sum_{j=1}^n \Eb{W_i^{IP}W_j^{IP} K(Z_i, Z_j)} 
+ \frac{1}{n^2}\sum_{i=1}^n\sum_{j=1}^n \Eb{K(Z_{n+i}, Z_{n+j})} 
\\ & \qquad - \frac{2}{n^2} \sum_{i=1}^n \sum_{j=1}^n \Eb{W_i^{IP} K(Z_i, Z_{n+j})}, 
\end{align*}
where for convenience $Z_i = (\bx_i, P_i)$ for $i =1, \ldots, 2n$.  

Next fix some $(i, j)$ with $i \neq j$. By \cref{IPWResult}, 
\[
\Eb{W_i^{IP}W_j^{IP} K(Z_i, Z_j)} \ = \ 
\Eb{W_j^{IP} K(Z_{n+i}, Z_j)}  \  = \ 
\Eb{K(Z_{n+i}, Z_{n+j})}.  
\]
Similarly, 
\[
	\Eb{W_i^{IP} K(Z_i, Z_{n+i})} \ = \ 
	\Eb{K(Z_{n+i}, Z_{n+i})}.
\]
Hence, substituting above, we see that all terms with $i\neq j$ drop out and we have that 
\begin{align*}
& \ \Eb{\bb(\bm W^{IP})^\top \bG \bb(\bm W^{IP})}\\
\ =& \ 
\frac{1}{n^2}\sum_{i=1}^n \Eb{(W_i^{IP})^2K(Z_i, Z_i)} 
+ \frac{1}{n^2}\sum_{i=1}^n\Eb{K(Z_{n+i}, Z_{n+i})} 
- \frac{2}{n^2} \sum_{i=1}^n \Eb{W_i^{IP} K(Z_i, Z_{n+i})}
\\
 \ =& \ 
\frac{1}{n^2}\sum_{i=1}^n \Eb{w_i^{IP}K(Z_{n+i}, Z_{n+i})} 
+ \frac{1}{n^2}\sum_{i=1}^n\Eb{K(Z_{n+i}, Z_{n+i})} 
- \frac{2}{n^2} \sum_{i=1}^n \Eb{K(Z_{n+i}, Z_{n+i})}
\\
\ =& \ 
\frac{1}{n^2}\sum_{i=1}^n \Eb{(W_i^{IP} - 1)K(Z_{n+i}, Z_{n+i})}, 
\end{align*}
by applying \cref{IPWResult} again.  

By assumption \textit{i)}, this last term is $O(1/n)$.  Thus, by Markov's inequality, the first term of \cref{eq:TargetExpansion_2} is $O_p(1/ n)$.  

For the second term of \cref{eq:TargetExpansion_2}, observe that 
\[
	\Eb{\frac{1}{n^2} \sum_{i=1}^n \left(W_i^{IP} P_j\right)^2} \ = \ 
	\frac{1}{ n} \left( \frac{1}{n}\sum_{i=1}^n \Eb{\left(W_i^{IP} P_i\right)^2}\right)
	\ = \ O(1/ n),
\]
by assumption \textit{ii)}.  Thus, by Markov's inequality, the second term of \cref{eq:TargetExpansion_2} is also $O_p(1/ n)$.

Combining these two pieces completes the proof.  
\end{proof}

\section{Proof of \cref{thm:bernstein_asymptotics}} \label{sec:bernstein}
\label{sec:asymptotics_of_bernstein_weights}




For convenience, define the functions 
\begin{align*}
q_{\max}(\bw):= \frac{1}{n} \max_{1 \leq i \leq n } \abs{w_i} p_i
\\
\text{WCBern}(\bw; \hat\Gamma)  \ := \ \max_{\br} \quad & {\bb(\bw)^\top(\br - \hat \br)} + \sqrt{2 \text{Var}(\bw; \br) \log(1/\epsilon)} + \frac{q_{\max}(\bw) \log(1/\epsilon)}{3}.
\\
\text{s.t.} \quad & \bm 0 \leq \br \leq \bpi ,\ \ \ \ (\br - \hat \br)^\top \bG^{-1}(\br - \hat \br) \leq \hat\Gamma^2.
\end{align*}

Our proof technique follows \cref{thm:mse_asymptotics} closely.


\begin{lemma} \label{lem:RelatingBernToWCBern}
For any $\bw$, 
\(
	\max(\text{Bern}(\bw, \hat \br + \bDelta), 0) \ \leq \ \max\left(1, \frac{\Gamma}{\hat\Gamma}\right) \cdot \text{WCBern}(\bw; \hat\Gamma).
\)
\end{lemma}
\begin{proof}
Notice by considering the feasible solution $\br = \hat \br$ that $\text{WCBern}(w, \hat\Gamma) \geq \sqrt{2\log(1/\epsilon) \text{Var}(\bw, \hat \br)} + \frac{q_{\max}(\bw) \log(1/\epsilon)}{3} \geq 0$.  Hence, when $\text{Bern}(\bw, \hat \br + \bDelta) \leq 0$, the inequality is trivially satisfied.  

Similarly, when $\Gamma \leq \hat\Gamma$, 
$\br$ feasible in the inner maximization defining $\bw^B(\hat\Gamma)$, so that $\text{Bern}(\bw^B(\hat\Gamma); \br) \leq \text{WCBern}(\bw^B(\hat\Gamma); \hat\Gamma)$.  

We thus focus on the case when $\Gamma > \hat\Gamma$ and $\text{Bern}(\bw, \hat \br + \bDelta) \geq 0$.  Then, 
\begin{align*}
\text{Bern}(\bw, \hat \br + \bDelta) 
& \ = \ 
\bb(\bw)^\top \bDelta + \sqrt{2 \log(1/\epsilon) \text{Var}(\bw; \hat \br + \bDelta)} + \frac{q_{\max}(\bw) \log(1/\epsilon)}{3}
\\
& \ \leq \ 
\frac{\Gamma}{\hat\Gamma} \left( 
\bb(\bw)^\top \bDelta\frac{\hat\Gamma }{\Gamma} + \frac{q_{\max}(\bw) \log(1/\epsilon)}{3} \right) + \sqrt{2 \log(1/\epsilon) \text{Var}(\bw; \hat \br + \bDelta)}, 
\end{align*}
since $\Gamma/\hat\Gamma > 1$ and $q_{\max} \geq 0$ by construction.  

Now consider the variance term.  From the proof of \cref{lem:BiasVar},
\begin{align*}
\text{Var}(\bw, \hat \br + \bDelta) 
& \ = \ 
\frac{1}{n^2}\sum_{i=1}^n w_i^2 (\hat r_i + \Delta_i)(p_i - \hat r_i - \Delta_i). 
\end{align*}
Since $\hat r_i\geq 0$, and $\hat r_i \leq p_i$, 
\[
	\hat r_i + \Delta_i \leq \frac{\Gamma}{\hat\Gamma}\left(\hat r_i + \frac{\hat\Gamma}{\Gamma}\Delta_i\right), 
\quad 
\text{ and } 
\quad
	p_i - \hat r_i - \Delta_i \ \leq \ 
	\frac{\Gamma}{\hat\Gamma}\left(p_i - \hat r_i - \frac{\hat\Gamma}{\Gamma}\Delta_i\right).
\]
Substituting above shows that 
\(
	\text{Var}(\bw, \hat \br + \bDelta) \ \leq \ \frac{\Gamma^2}{\hat\Gamma^2} \text{Var}(\bw, \hat \br + \frac{\hat\Gamma}{\Gamma}\bDelta).
\)
In summary, we have shown that 
\[
	\text{Bern}(\bw; \hat \br + \bDelta) \ \leq \frac{\Gamma}{\hat\Gamma} \text{Bern}(\bw; \hat \br + \frac{\hat\Gamma}{\Gamma}\bDelta).
\]
To complete the proof, note that $\hat \br + \frac{\hat\Gamma}{\Gamma}\bDelta$ is feasible in the optimization defining $\text{WCBern}(\bw; \hat\Gamma)$. 
\end{proof}

We can now prove our main result.  

\begin{proof}[Proof of Theorem~\ref{thm:bernstein_asymptotics}]
From \cref{lem:RelatingBernToWCBern}, it suffices to show that $\text{WCBern}(\bw^B(\hat\Gamma); \hat\Gamma) = O_p(1/\sqrt n)$.  We show this latter claim by relating $\bw^B(\hat\Gamma)$ with the scaled inverse propensity weights $\bm W^{IP}/n$.

Specifically, since $\bm W^{IP}/n$ is feasible in the outer optimization problem defining $\bw^B(\hat\Gamma)$ we have that
\begin{align*}
 &\text{WCBern}(\bw^B(\hat\Gamma); \hat\Gamma) 
 \\ \qquad & \ \leq \ 
 \text{WCBern}(\bm W^{IP}/n; \hat\Gamma)
 \\
& \ \leq \ 
\max_{\br : (\br - \hat \br)^\top \bG^{-1}(\br - \hat \br) \leq \hat\Gamma^2} \text{Bern}(\bm W^{IP}/n; \br)
\\
& \ \leq \ 
\max_{\br : (\br - \hat \br)^\top \bG^{-1}(\br - \hat \br) \leq \hat\Gamma^2}  \ \bb(\bm W^{IP}/n)^\top(\br - \hat \br) + \frac{\sqrt{2 \log(1/\epsilon)}}{4} \sqrt{\frac{1}{n^2}\sum_{i=1}^n (W_i^{IP} P_i)^2 } + \frac{q_{\max}(\bm W^{IP}/n) \log(1/\epsilon)}{3},
\end{align*} 
where the second to last inequality follows by expanding the feasible region and the last by upper bounding the variance since $d_j(1-d_j) \leq \frac{1}{4}$.  We evaluate the maximization in closed form and round the constants up to $1$ yielding

\begin{equation} \label{eq:TargetExpansion}
	\text{WCBern}(\bw^B(\hat\Gamma); \hat\Gamma) \ \leq \ 
	\hat\Gamma \sqrt{\bb(\bm W^{IP})^\top \bG \bb(\bm W^{IP})} + \sqrt{\log(1/\epsilon)} \sqrt{\frac{1}{n^2}\sum_{i=1}^n (W_i^{IP}P_i)^2 } + 
	q_{\max}(\bm W^{IP}/n) \log(1/\epsilon).  
\end{equation}

We tackle the first term by upper bounding its expectation and applying Markov's inequality.  Specifically, $\Eb{\sqrt{\bb(\bm W^{IP})\top \bG \bb(\bm W^{IP})}} \ \leq \ \sqrt{\Eb{\bb(\bm W^{IP})\top \bG \bb(\bm W^{IP})}}$ by Jensen's inequality.  



Following an identical argument to that in \cref{thm:mse_asymptotics} which uses assumption \textit{i)}, we have that 
\(
\Eb{\bb(\bm W^{IP})^\top \bG \bb(\bm W^{IP})} = O(1/n).
\)
Thus, by Markov's inequality, the first term of \cref{eq:TargetExpansion} is $O_p(1/\sqrt n)$.  

For the second term of \cref{eq:TargetExpansion}, observe again by Jensen's inequality that 
\[
	\Eb{\sqrt{\frac{1}{n^2} \sum_{i=1}^n \left(W_i^{IP} P_i\right)^2}} \ \leq \ 
	\frac{1}{\sqrt n} \sqrt{ \frac{1}{n}\sum_{i=1}^n \Eb{\left(W_i^{IP} P_i\right)^2}} 
	\ = \ O(1/\sqrt{ n}),
\]
again by assumption \textit{ii)}.  Thus, by Markov's inequality, the second term of \cref{eq:TargetExpansion} is also $O_p(1/\sqrt n)$.

Finally, for the last term, observe that 
\begin{align*}
q_{\max}(\bm W^{IP}/n) = \frac{1}{n} \max_i \abs{W_i^{IP}P_i} \ \leq \ \frac{1}{n} \sqrt{\sum_{i=1}^n \left(W_i^{IP}P_i\right)^2},
\end{align*}

since the $\ell_2$-norm bounds the $\ell_\infty$-norm.  Taking expectations and applying the above inequality with Markov's inequality shows the last term is also $O_p(1/\sqrt n)$.  

Combining these three pieces completes the proof.  
\end{proof}

\section{Additional Experiments} \label{sec:nomis}
\label{sec:additional_exps}
We present results for 3 more subsets of the Nomis dataset. Apart from the subset of data used, the experiment set up are same as that in Table~\ref{table:nomis_mse_1}.

\begin{table}[htb!]
\centering
\begin{tabular}{@{}lllll@{}}
\toprule
\textbf{Metrics} & \textbf{BOPE-B} & \textbf{BOPE} & \textbf{LASSO} & \textbf{SPPE} \\ \midrule
\multicolumn{5}{c}{Target Policy is 5\% increase.} \\
\midrule
  MSE & 0.60 & 0.69 & 0.90 & 3.15\\
  Bias$^2$ & 0.28 & 0.32 & 0.43 & 1.51\\
  Variance & 0.32 & 0.37 & 0.47 & 1.64\\
\midrule
\multicolumn{5}{c}{Target Policy is 5\% decrease.} \\
\midrule
  MSE & 0.09 & 0.09 & 0.18 & 0.71\\
  Bias$^2$ & 0.04 & 0.04 & 0.09 & 0.34\\
  Variance & 0.05 & 0.05 & 0.09 & 0.37\\
\midrule
\multicolumn{5}{c}{Target Policy is 10\% increase.} \\
\midrule
  MSE & 1.77 & 1.94 & 2.22 & 5.78\\
  Bias$^2$ & 0.84 & 0.92 & 1.06 & 2.79\\
  Variance & 0.93 & 1.02 & 1.16 & 2.99\\
\midrule
\multicolumn{5}{c}{Target Policy is 10\% decrease.} \\
\midrule
  MSE & 0.02 & 0.02 & 0.03 & 0.08\\
  Bias$^2$ & 0.01 & 0.01 & 0.01 & 0.03\\
  Variance & 0.01 & 0.01 & 0.02 & 0.05\\
  \bottomrule
\end{tabular}
\caption{Decomposition of the mean squared error. Tier = 2, Car Type = Used, Term = 60, Partner Bin = 1, and year 2003. There are 609 datapoints in the subset.}
\label{table:nomis_mse_appendix_1}
\end{table}

\begin{table}[htb!]
\centering
\begin{tabular}{@{}lllll@{}}
\toprule
\textbf{Metrics} & \textbf{BOPE-B} & \textbf{BOPE} & \textbf{LASSO} & \textbf{SPPE} \\ \midrule
\multicolumn{5}{c}{Target Policy is 5\% increase.} \\
\midrule
  MSE & 0.54 & 0.62 & 0.61 & 0.64\\
  Bias$^2$ & 0.12 & 0.14 & 0.15 & 0.06\\
  Variance & 0.42 & 0.48 & 0.46 & 0.59\\
\midrule
\multicolumn{5}{c}{Target Policy is 5\% decrease.} \\
\midrule
  MSE & 0.47 & 0.54 & 0.48 & 0.47\\
  Bias$^2$ & 0.10 & 0.10 & 0.11 & 0.08\\
  Variance & 0.37 & 0.44 & 0.37 & 0.39\\
\midrule
\multicolumn{5}{c}{Target Policy is 10\% increase.} \\
\midrule
  MSE & 0.98 & 1.13 & 1.10 & 1.23\\
  Bias$^2$ & 0.31 & 0.37 & 0.37 & 0.22\\
  Variance & 0.67 & 0.76 & 0.73 & 1.01\\
\midrule
\multicolumn{5}{c}{Target Policy is 10\% decrease.} \\
\midrule
  MSE & 0.55 & 0.58 & 0.51 & 0.90\\
  Bias$^2$ & 0.11 & 0.09 & 0.10 & 0.25\\
  Variance & 0.44 & 0.49 & 0.41 & 0.65\\
  \bottomrule
\end{tabular}
\caption{Decomposition of the mean squared error. Tier = 3, Car Type = Used, Term = 72, Partner Bin = 3, and year 2002-2004. There are 667 datapoints in the subset.}
\label{table:nomis_mse_appendix_2}
\end{table}

\begin{table}[htb!]
\centering
\begin{tabular}{@{}lllll@{}}
\toprule
\textbf{Metrics} & \textbf{BOPE-B} & \textbf{BOPE} & \textbf{LASSO} & \textbf{SPPE} \\ \midrule
\multicolumn{5}{c}{Target Policy is 5\% increase.} \\
\midrule
  MSE & 0.39 & 0.52 & 0.49 & 1.11\\
  Bias$^2$ & 0.08 & 0.12 & 0.12 & 0.24\\
  Variance & 0.31 & 0.40 & 0.37 & 0.87\\
\midrule
\multicolumn{5}{c}{Target Policy is 5\% decrease.} \\
\midrule
  MSE & 0.66 & 0.81 & 0.78 & 0.38\\
  Bias$^2$ & 0.24 & 0.29 & 0.30 & 0.06\\
  Variance & 0.42 & 0.52 & 0.48 & 0.32\\
\midrule
\multicolumn{5}{c}{Target Policy is 10\% increase.} \\
\midrule
  MSE & 0.64 & 0.88 & 0.75 & 1.08\\
  Bias$^2$ & 0.20 & 0.25 & 0.25 & 0.17\\
  Variance & 0.44 & 0.63 & 0.50 & 0.91\\
\midrule
\multicolumn{5}{c}{Target Policy is 10\% decrease.} \\
\midrule
  MSE & 0.50 & 0.56 & 0.50 & 0.46\\
  Bias$^2$ & 0.14 & 0.15 & 0.15 & 0.10\\
  Variance & 0.36 & 0.41 & 0.35 & 0.36\\
  \bottomrule
\end{tabular}
\caption{Decomposition of the mean squared error. Tier = 3, Car Type = Used, Term = 60, Partner Bin = 3, and year 2002-2004. There are 1851 datapoints in the subset.}
\label{table:nomis_mse_appendix_3}
\end{table}

\vfill
\end{document}